%% file: main.tex
\documentclass[10pt,twocolumn,letterpaper]{article}

\usepackage{cvpr}              % To produce the CAMERA-READY version
\input{preamble}
\definecolor{cvprblue}{rgb}{0.21,0.49,0.74}
\definecolor{darkred}{rgb}{0.55, 0.0, 0.0}
\definecolor{darkgreen}{RGB}{1,50,32}

\usepackage{marvosym}
\usepackage{adjustbox}
\usepackage{graphicx}
\usepackage{amssymb}
\usepackage{pifont}
\usepackage{tabularx}
\usepackage[pagebackref,breaklinks,colorlinks,allcolors=cvprblue]{hyperref}
\usepackage{booktabs} % professional-quality horizontal rules
\usepackage{multirow} % for multi-row headers (used for "Method")
\usepackage{array}  % improved column formatting
\usepackage{makecell}
\usepackage{upquote} % for straight quotation marks `` instead of curved ones.
\usepackage{booktabs} % For professional tables (toprule, midrule, bottomrule)
\usepackage[table]{xcolor} % For colors (defining darkred/darkgreen)

\definecolor{darkred}{rgb}{0.6, 0, 0}
\definecolor{darkgreen}{rgb}{0, 0.6, 0}
\usepackage{url}
\usepackage[utf8]{inputenc} % allow utf-8 input
\usepackage[T1]{fontenc}    % use 8-bit T1 fonts
\usepackage{amsmath}
\usepackage{amsfonts}       % blackboard math symbols
\usepackage{nicefrac}       % compact symbols for 1/2, etc.

\usepackage{subfiles}
\usepackage{microtype}      % microtypography
\usepackage{wrapfig}
\usepackage{arydshln}
\usepackage{placeins}
\usepackage{svg}

\usepackage{amsthm}
\usepackage{enumitem}
\usepackage{comment}
\usepackage[font=small,labelfont=bf]{caption}
\usepackage{subcaption}
\newtheorem{remark}{Remark}[section]

\newcommand{\EE}{\mathbb{E}}
\appendix
\setlist[itemize]{leftmargin=0pt, labelsep=1em, itemsep=0pt, parsep=0pt}
\newtheorem{assumption}{Assumption}[section]
\newtheorem{proposition}{Proposition}[section]
\usepackage{xr}
\usepackage{siunitx}

\usepackage[table]{xcolor}
\usepackage{colortbl}          % extra safety
\definecolor{lightblue1}{RGB}{222,235,247} % lightest blue
\definecolor{lightblue2}{RGB}{189,215,238} % medium blue
\definecolor{lightblue3}{RGB}{158,202,225} % slightly darker blue

\newcolumntype{G}{>{\color{gray}}c}

\definecolor{cvprblue}{rgb}{0.21,0.49,0.74}
\usepackage[pagebackref,breaklinks,colorlinks,allcolors=cvprblue]{hyperref}

\def\confName{CVPR}
\def\confYear{2026}

\title{Training a Student Expert via Semi-Supervised Foundation Model Distillation}
\author{%
  Pardis Taghavi, Tian Liu, Renjie Li, Reza Langari, Zhengzhong Tu\thanks{Corresponding author}\\
  Texas A\&M University\\
  \texttt{\{ptgh,ltmask,renjie,rlangari,tzz\}@tamu.edu} \\
}

\begin{document}
\maketitle
\input{sec/0_abstract}    
\input{sec/1_intro}

\input{sec/2_related_work}
\input{sec/3_method}
\input{sec/4_experiments}

\input{sec/5_conclusion}

\input{sec/X_suppl}

{
    \small
    \bibliographystyle{ieeenat_fullname}
    \bibliography{main}
}

% WARNING: do not forget to delete the supplementary pages from your submission 
%\input{sec/X_suppl}

\end{document}

%% file: sec/0_abstract.tex
\begin{abstract}

Foundation models deliver strong perception but are often too computationally heavy to deploy, and adapting them typically requires costly annotations. We introduce a semi-supervised knowledge distillation (SSKD) framework that compresses pre-trained vision foundation models (VFMs) into compact experts using limited labeled and abundant unlabeled data, and instantiate it for instance segmentation where per-pixel labels are particularly expensive. The framework unfolds in three stages: (1) domain adaptation of the VFM(s) via self-training with contrastive calibration, (2) knowledge transfer through a unified multi-objective loss, and (3) student refinement to mitigate residual pseudo-label bias. Central to our approach is an instance-aware pixel-wise contrastive loss that fuses mask and class scores to extract informative negatives and enforce clear inter-instance margins. By maintaining this contrastive signal across both adaptation and distillation, we align teacher and student embeddings and more effectively leverage unlabeled images. On Cityscapes and ADE20K, our $\approx 11\times$ smaller student improves over its zero-shot VFM teacher(s) by +11.9 and +8.6 AP, surpasses adapted teacher(s) by +3.4 and +1.5 AP, and outperforms state-of-the-art SSKD methods on benchmarks.

\end{abstract}

%% file: sec/1_intro.tex
\section{Introduction}
\label{sec:intro}

Vision foundation models (VFMs)~\cite{oquab2023dinov2, liu2024grounding, yuan2025sa2va, kirillov2023segany} have substantially expanded the capabilities of computer vision systems, achieving strong performance across diverse perception benchmarks~\cite{awais2025foundation}. However, their scale often makes deployment costly or impractical in resource constrained settings, and their generic training objectives can yield suboptimal performance on specialized downstream tasks.
Instance segmentation amplifies challenges: pixel-level mask annotation is expensive, and state-of-the-art instance segmentation models can require substantial training compute~\cite{Cordts2016Cityscapes, he2017mask}.

\noindent\textbf{Motivation.}
Despite remarkable progress, foundation models still face important challenges in task- and domain-specific instance segmentation due to two recurring issues:
(1) heavy computational overhead at inference time, which limits real world deployment under strict latency, memory, or energy budgets~\cite{xu2024survey, zhuang2025argus}; and
(2) limited specialization, as models optimized to transfer broadly can underperform on domains-specific tasks~\cite{sony2025foundation, bommasani2021opportunities, lu2025general}.
%This challenge is particularly prominent in outdoor environments such as autonomous driving and in indoor settings such as robotic perception~\cite{firoozi2023foundation, yan2024forging}.
This need for specialized, efficient models is particularly evident in outdoor applications such as autonomous driving and in indoor settings such as robotic perception~\cite{firoozi2023foundation, yan2024forging, tao2026navidrivevlm}.
%We focus on two representative application domains: outdoor environments such as autonomous driving and indoor settings such as robotic perception~\cite{firoozi2023foundation, yan2024forging}.

Semi-supervised knowledge distillation (SSKD) offers a practical approach for instance segmentation, compressing a large teacher into an efficient student while leveraging limited labeled data together with abundant unlabeled images.
However, existing approaches either treat VFMs as fixed feature extractors~\cite{Jang2025vl2, lee2025customkd}, focus on coarser semantic tasks~\cite{liang2025task}, or, when targeting instance segmentation, fail to fully exploit the structure of unlabeled data to refine dense mask predictions~\cite{yoon2025s, lin2025pseudo}.
As a result, adjacent instances remain poorly separated and performance degrades in low-label regimes.

We address these limitations with a stage-wise training paradigm that (i) adapts the VFM(s) via self-training to improve pseudo-label alignment, and (ii) introduces an instance-aware pixel-wise contrastive loss that uses unlabeled data to enforce clear inter-instance margins.
Crucially, this self-supervised contrastive signal is maintained across both teacher adaptation and student distillation, improving mask separation and yielding stronger performance under limited supervision.

\noindent\textbf{Status quo.}
Knowledge distillation has evolved from task-agnostic compression~\cite{hinton2015distilling,chen2020big} to adapting VFMs for downstream tasks.
For classification and semantic segmentation, Vemulapalli et al.~\cite{Vemulapalli2024KD} distill a VFM by matching its outputs on an unlabeled transfer set, and SAM-CLIP~\cite{wang2024sam} fuses CLIP and SAM.
However, neither targets per-pixel instance masks nor leverages dense self-supervision from the unlabeled pool for mask refinement.
Pure semi-supervised instance segmentation methods~\cite{hu2023pseudo, berrada2024guided} often train teachers from scratch, increasing compute cost, and can still produce noisy masks under scarce labels,  leaving the potential of modern foundation models underexploited.
%To our knowledge, no prior work unifies VFM adaptation, unlabeled data-driven pixel-wise refinement, and extreme student compression for instance segmentation.

\textbf{Contributions.}
We summarize our main contributions as follows:
\begin{itemize}[nosep,leftmargin=6mm]
  \item We introduce an \emph{instance-aware pixel-wise contrastive loss} that combines mask and class predictions to identify informative negatives and enforce stronger inter-instance separation in dense prediction settings.

  \item We propose a three-stage semi-supervised foundation model distillation framework for training compact student experts: (i) domain adaptation of the foundation teacher via self-training with contrastive calibration, (ii) distillation into a compact student using a unified objective over labeled and unlabeled data, and (iii) student refinement to mitigate residual pseudo-label bias.

  \item We validate the proposed framework on Cityscapes and ADE20K. Although the student is approximately $11\times$ smaller than the teacher, it improves over the zero-shot VFM teacher by +11.9 AP on Cityscapes and +8.6 AP on ADE20K, and exceeds the adapted teacher by +3.4 AP and +1.5 AP, respectively. Across both benchmarks, it also outperforms prior semi-supervised knowledge distillation baselines.
\end{itemize}

%% file: sec/2_related_work.tex
\section{Related work}
\label{sec:Related Work}

\noindent\textbf{Vision Foundation Models.}
Vision foundation models (VFMs)~\cite{oquab2023dinov2, liu2024grounding, ravi2024sam, yang2024depth, bochkovskii2024depth} have substantially advanced computer vision through large scale pre-training and strong transferability across diverse tasks. Recent efforts further extend their capabilities by combining complementary VFMs~\cite{ren2024grounded,yuan2025sa2va}. Despite their strong open-set recognition and transfer performance, these models remain computationally demanding, which limits deployment in resource-constrained settings. To address this challenge, recent works explore compressing or merging VFMs via distillation. For example, Wang et al.~\cite{wang2024sam} unify SAM and CLIP through multi-task learning, while Zhang et al.~\cite{zhang2025accessing} distill CLIP and DINOv2 into a compact model using moderate-scale data distillation. We build on this line of work by studying how VFMs can supervise compact student experts for instance segmentation under limited labeled data.

\noindent\textbf{Knowledge Distillation in Vision.}
Knowledge distillation (KD) addresses transferring knowledge from high-capacity teachers to lightweight students for efficient deployment. Early methods distill softened logits or intermediate representations in a task-agnostic manner~\cite{hinton2015distilling}, while later feature-based approaches capture richer spatial and channel-wise structure~\cite{rajasegaran2020self, shu2021channel}. More recent work studies distillation from large pre-trained or foundation models~\cite{sun2023dime, yang2024clip}, including multi-teacher settings that combine complementary expertise~\cite{jiang2024mtkd, yang2025multi}. Vemulapalli et al.~\cite{Vemulapalli2024KD} adapt a VFM to the target task and distill it on an unlabeled transfer set for classification and semantic segmentation. We focus on instance segmentation, where compact deployment is desirable but mask-level supervision is expensive, and where unlabeled data must be exploited at finer spatial granularity.

A complementary line of work studies contrastive knowledge distillation, where teacher and student representations are aligned through contrastive objectives~\cite{tian2019contrastive, fang2021seed, zhu2021complementary}. Extensions to dense prediction have explored ROI- or pixel-level contrastive distillation for object detection and semantic segmentation~\cite{yao2021g, yang2022cross, fan2023augmentation, huang2023pixel}. In contrast, our method does not rely on explicit teacher--student contrastive matching; instead, it uses an instance-aware pixel-wise contrastive objective as a self-supervised signal within a unified semi-supervised distillation pipeline for instance segmentation.

\noindent\textbf{Semi-Supervised Learning.}
Self‐training (or pseudo‐labeling) has become a foundational paradigm in semi‐supervised learning (SSL), where a model leverages its own predictions with high confidence and iteratively refines itself~\cite{xie2020self}. This approach has proven effective across vision tasks, improving image classification performance~\cite{xie2020self} and boosting object detection accuracy when annotation budgets are tight~\cite{liu2021unbiased}. To counteract error accumulation from noisy pseudo‐labels~\cite{tarvainen2017mean} use exponential moving average of label predictions, or~\cite{cascante2021curriculum} employ curriculum labeling schemes that gradually incorporate harder examples. More recent work applies pseudo-labeling for large pre-trained models through targeted finetuning and adaptive pseudo selection strategies~\cite{gan2024erasing}. While many SSL methods focus on classification or detection, several have extended this method to dense prediction tasks~\cite{chen2021semi, yang2023revisiting}.

We study self‐training with self‐supervised contrastive learning and task‐specific adaptation. Global contrastive frameworks such as SimCLR~\cite{chen2020simple}, MoCo~\cite{chen2021empirical}, and their detection extensions~\cite{xie2021detco} established the value of large-scale visual discrimination learning. Further per-pixel contrastive approaches~\cite{wang2021dense, xie2021propagate, zhong2021pixel, wang2022contrastmask, alonso2021semi} have shown promise in retaining spatial sensitivity though they yet conflate pixels from different instances of the same class. We extend these advances by synergizing self-training and self-supervised contrastive learning, and introduce a novel instance-aware negative sampling strategy designed specifically for the demands of instance segmentation.

\begin{figure*}[t]
\begin{center}
   \includegraphics[width=1\linewidth]{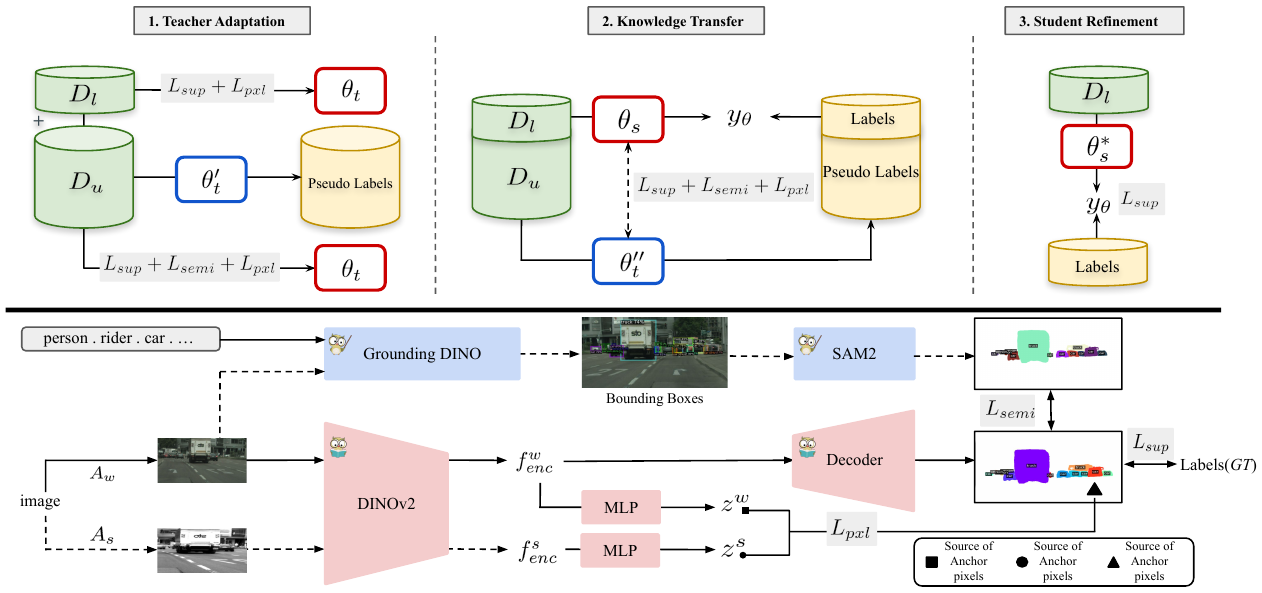}
   \end{center}
\caption{%
  \textbf{Framework overview.} 
  \textbf{Top:} Three-stage pipeline: 
  (1) adapt a pre-trained VFM teacher to the target domain via self-training with pixel-level contrastive calibration; 
  (2) distill knowledge into a compact student using instance-aware contrastive sampling; 
  (3) fine-tune the student on labeled data to correct residual pseudo-label bias.  
  \textbf{Bottom:} Detailed view of stage (2): fused mask and class score maps produce anchor pixels, sampled across weak/strong views to form positive/negative pairs; an MLP projects features for the contrastive loss. Dashed arrows denote no gradient flow; red modules are trainable, blue are frozen.
}

\label{fig:setup}
\end{figure*}

%% file: sec/3_method.tex
\section{Method}
\subsection{Overview}
In semi‐supervised settings, we are given a small labeled set and a substantially larger unlabeled pool:
\[
  \mathcal{D}^l = \bigl\{(x_i^l, y_i^l)\bigr\}_{i=1}^{N_l}
  \quad\text{and}\quad
  \mathcal{D}^u = \bigl\{x_i^u\bigr\}_{i=1}^{N_u},
  \quad N_u \gg N_l,
\]
where each annotation \(y_i^l\) consists of binary masks and class labels for every instance. Our goal is to transfer knowledge from a large, pretrained VFM into a compact student \(f_{\theta_s}\) that achieves comparable or better accuracy with substantially lower computational cost.
 We propose a three-stage SSKD pipeline that hinges on two key ideas: \ding{182} \underline{\emph{Contrastive Calibration.}}  We fine-tune a large VFM teacher via self-training, but rather than simple pseudo-labels we also inject a pixel-wise contrastive head to sharpen mask boundaries. 
\ding{183} \underline{\emph{Debiased, Instance-Aware Sampling.}}  During both adaptation and distillation, we mine hard negatives via a joint mask-/class-probability embedding, focusing repulsion on informative inter-instance pairs tailored for instance segmentation.
These two ideas are then realized in three concise stages (see Fig.~\ref{fig:setup}):
\begin{enumerate}[leftmargin=6mm,nosep]
  \item \textbf{Teacher Adaptation.}  We adapt the pretrained VFM teacher using labeled data, pseudo-labels on unlabeled data, and pixel-wise contrastive regularization.  
  \item \textbf{Knowledge Transfer.}  We freeze the adapted teacher and distill its knowledge into a lightweight student using supervised, pseudo-label, and contrastive objectives.  
  \item \textbf{Student Refinement.}  We fine-tune the student on labeled data only to reduce residual bias introduced by pseudo-labels.
\end{enumerate}
 
Sec.~\ref{sec:contrastive-loss} formalizes our instance-aware pixel-wise contrastive loss, which is used in both Teacher Adaptation and Knowledge Transfer to enforce intra-instance cohesion and inter-instance separation; Sec.~\ref{sec:pipeline} then details the three stages of the training pipeline.
%%%%%%%%%%%%%%%%%%%%%%%%%
\subsection{Contrastive Loss}
\label{sec:contrastive-loss}
Standard supervised and pseudo-label losses enforce correct mask predictions, but they do not explicitly model pixel-level feature relationships. As a result, they underutilize unlabeled data and can amplify pseudo-label noise. To better exploit both labeled and unlabeled images, we add a pixel-wise contrastive loss that improves feature discrimination and regularizes training against noisy supervision.

Let \(z^{\rm weak}, z^{\rm strong} \in \mathbb{R}^{B \times N \times D}\) be \(\ell_2\)-normalized embeddings extracted from two augmented views of each image, where \(B\) is the batch size, \(N = h \times w\) is the number of spatial locations at the feature resolution, and \(D\) is the embedding dimension. For image index \(b \in \{1, \dots, B\}\) and spatial location \(p \in \{1, \dots, N\}\), the corresponding embedding vector is denoted by \(z_{b,p} \in \mathbb{R}^D\).

For each anchor pixel, the positive pair is formed by matching the embeddings at the same spatial location in the weak and strong views. The positive similarity is defined as
\[
s_{b,p}^{+} = \langle z^{\rm weak}_{b,p}, z^{\rm strong}_{b,p}\rangle / T,
\]
where \(T\) is a temperature parameter.

Negatives are sampled using our instance-aware sampler (Sec.~\ref{par:debiased-sampling}), which returns indices \(\{(b', q_r)\}_{r=1}^R\). The corresponding negative similarities are
\[
s_{b,p,r}^{-}
=
\langle z^{\rm weak}_{b,p}, z^{\rm strong}_{b',q_r}\rangle / T,
\qquad r=1,\dots,R.
\]

The pixel-wise contrastive loss is then defined as the standard NT-Xent objective over all anchor pixels:
\[
\mathcal{L}_{\rm pxl}
=
-\frac{1}{BN}
\sum_{b=1}^{B}\sum_{p=1}^{N}
\log
\frac{\exp(s_{b,p}^{+})}
{\exp(s_{b,p}^{+}) + \sum_{r=1}^{R}\exp(s_{b,p,r}^{-})}.
\]
%%%%%%%%
\noindent\textbf{Debiased Pixel-Level Negative Sampling.}
\label{par:debiased-sampling}

To efficiently sample informative negatives for pixel-wise contrastive learning, we construct a per-pixel sampling distribution by fusing mask and class predictions. The goal is to favor pixels that are more likely to belong to different instances, without incurring quadratic pairwise comparisons.
Let
\[
M \in \mathbb{R}^{B \times K \times H \times W},
\qquad
L \in \mathbb{R}^{B \times K \times (C+1)},
\]
denote the model's mask logits and class logits, respectively, where \(K\) is the number of candidate instances and \(C\) is the number of foreground classes. We first resize \(M\) to the feature resolution \((h \times w)\), and then convert the mask and class logits into probability distributions \(P_m\) and \(P_c\) by applying softmax over the instance and class dimensions, respectively.
For each pixel \((b,p)\), we compute the expected class distribution
\begin{align}
F_c[b,p,c] &= \sum_{k=1}^K P_m[b,k,p]\,P_c[b,k,c]. \label{eq:f_c}
\end{align}
The distribution \(F_c\) captures semantic class information at each pixel, but it may blur instance identity when multiple instances share the same class. To preserve both instance-level and class-level cues, we form a joint pseudo-probability embedding by concatenating the mask distribution and the expected class distribution:
\begin{align}
y[b,p] &=
\begin{bmatrix}
P_m[b,1\!:\!K,p]\\[4pt]
F_c[b,p,1\!:\!C\!+\!1]
\end{bmatrix}
\in \mathbb{R}^{K + (C+1)}. \label{eq:y_bp}
\end{align}

Let \(\tilde y[b,p]\) denote the \(\ell_2\)-normalized version of \(y[b,p]\). We define the dissimilarity score between two pixels \((b,p)\neq(b',q)\) as
\[
s^{\mathrm{deb}}\bigl((b,p),(b',q)\bigr)
=
\max\bigl(0,\;1 - \langle \tilde y[b,p], \tilde y[b',q]\rangle\bigr).
\]
Pixels with larger dissimilarity scores are more likely to correspond to different instances and therefore serve as more informative negatives. For each anchor pixel \((b,p)\), we sample \(R\) negatives \(\{q_r\}\) proportionally to \(s^{\mathrm{deb}}\), and use them in the denominator of the NT-Xent loss \(\mathcal{L}_{\rm pxl}\).
%%%%
\\

\noindent\textbf{Theoretical Insight.}
%To give a formal rationale for augmenting our pixel‐wise contrastive loss, we show that even under a mild negative sampling guarantee, each gradient step on our contrastive term provably increases the expected inter-instance margin.

Our contrastive term is motivated by the intuition that better negative sampling should improve separation between different instances in the learned feature space. The following result formalizes this intuition under a mild assumption on the quality of sampled negatives.

\begin{assumption}[Negative Sampling Guarantee]\label{assump:sampling}
When sampling a negative under our instance aware scheme, the probability it originates from a different instance is at least $p>0.5$, where $p$ can be estimated empirically (see Sec.~\ref{emp}).
\end{assumption}
\begin{proposition}[Expected Margin Growth]\label{prop:margin-growth}
Under Assumption \ref{assump:sampling}, one gradient update on $\mathcal L_{\rm pxl}$ increases the expected inter-instance margin $\Delta_{\rm emp}$ by
\[
\varepsilon = \Theta(p\,\lambda_{\rm pxl}) > 0.
\]
This expectation holds even when pseudo-labels are imperfect, provided negatives are sampled using our instance aware strategy.
\end{proposition}
In practice, raising $\lambda_{\rm pxl}$ enhances margin growth but also increases training cost. If $\lambda_{\rm pxl}$ is too large, it can overemphasize inter-instance separation at the expense of intra-instance cohesion. We validate this effect in Sec.~\ref{emp} and provide a proof sketch in Appendix C.

\subsection{Training Framework}
\label{sec:pipeline}
We formulate teacher adaptation, student distillation, and student refinement as special cases of a unified objective with three terms.  Let
\begin{equation}
\begin{split}
\mathcal{J}(\theta;&\mathcal{D}^l,\mathcal{D}^u;\lambda_{\rm semi},\lambda_{\rm pxl})= \underbrace{\frac{1}{N_l}\sum_{i=1}^{N_l}\ell\bigl(f_\theta(x_i^l),y_i^l\bigr)}_{\mathcal{L}_{\rm sup}} \\
&+ \lambda_{\rm semi}\underbrace{\frac{1}{N_u}\sum_{j=1}^{N_u}\ell\bigl(f_\theta(x_j^u),\hat y_j^u\bigr)}_{\mathcal{L}_{\rm semi}} + \lambda_{\rm pxl}\,\mathcal{L}_{\rm pxl}\bigl(\theta;\mathcal{D}^l\cup\mathcal{D}^u\bigr),
\end{split}
\end{equation}

where $\mathcal{D}^u=\varnothing$, the semi-supervised term vanishes.

\noindent\textbf{Teacher adaptation.}
\label{sec:domain-adaptation}
Starting from pretrained teacher weights \(\theta_T^0\), we first fine-tune the model on the labeled set \(\mathcal{D}^l\):
\[
\theta_T'
=
\arg\min_{\theta}\;
\mathcal{J}\bigl(\theta;\mathcal{D}^l,\varnothing;0,\lambda_{\rm pxl}\bigr).
\]
We then generate pseudo-labels for the unlabeled set using the adapted teacher,
\(
\hat y_j^u = f_{\theta_T'}(x_j^u),
\)
and reinitialize from \(\theta_T^0\) before training on both labeled and pseudo-labeled data:
\[
\theta_T''
=
\arg\min_{\theta}\;
\mathcal{J}\bigl(\theta;\mathcal{D}^l,\mathcal{D}^u;1,\lambda_{\rm pxl}\bigr).
\]
This two-stage procedure yields a teacher that is better adapted to the target domain and produces pseudo-labels that are both more accurate and more spatially consistent.

\noindent\textbf{Knowledge transfer.}
\label{sec:dist}
With the adapted teacher \(\theta_T''\) frozen, we train the student \(\theta_s\) using the same objective:
\begin{equation}
\label{eq:student-distill}
\theta_s^*
=
\arg\min_{\theta_s}\;
\mathcal{J}\bigl(\theta_s;\mathcal{D}^l,\mathcal{D}^u;\lambda_{\rm semi},\lambda_{\rm pxl}\bigr).
\end{equation}
Here, \(\mathcal{L}_{\rm sup}\) provides ground-truth supervision on \(\mathcal{D}^l\), \(\mathcal{L}_{\rm semi}\) transfers pseudo-label knowledge on \(\mathcal{D}^u\), and \(\mathcal{L}_{\rm pxl}\) imposes pixel-wise contrastive regularization across both sets. The coefficients \(\lambda_{\rm semi}\) and \(\lambda_{\rm pxl}\) balance these signals, enabling the student to match or surpass the teacher with far fewer parameters.

\noindent\textbf{Student refinement.}
\label{sec:student-refinement}
Although joint distillation yields a strong initialization, residual pseudo-label noise and contrastive regularization can still introduce bias. As a final step, we fine-tune the student on labeled data only:
\[
\theta_s^{\dagger}
=
\arg\min_{\theta}\;
\mathcal{J}\bigl(\theta;\mathcal{D}^l,\varnothing;0,0\bigr),
\qquad \text{initialized from } \theta_s^*.
\]
This refinement stage reduces pseudo-label drift and sharpens decision boundaries for the target domain.

%% file: sec/4_experiments.tex
\section{Experiments}
\subsection{Experimental Protocol}
\noindent\textbf{Datasets.}
We evaluate our method on two standard instance segmentation benchmarks. \textbf{Cityscapes}~\cite{Cordts2016Cityscapes} contains 2,975 training images and 500 validation images of urban street scenes, annotated with 19 semantic categories, including 8 ``thing'' classes and 11 ``stuff'' classes. \textbf{ADE20K}~\cite{zhou2019semantic} comprises 20,210 training images and 2,000 validation images spanning diverse indoor and outdoor scenes, annotated with 150 semantic categories, including 100 ``thing'' classes and 50 ``stuff'' classes.

\noindent\textbf{Implementation Details.}
All experiments were conducted on Ubuntu~22.04 using Python~3.10 and PyTorch~2.6.0 with CUDA~12.6. Teacher adaptation runs were executed on 2$\times$NVIDIA A100 GPUs, while student training runs were performed on 2$\times$NVIDIA GeForce RTX 4090 GPUs. As a reference point, a single supervised fine-tuning run of the teacher (Grounding-DINO) on the Cityscapes labeled split required approximately 3.5 GPU-hours, whereas a single student training run on the same dataset required approximately 17 GPU-hours.

%I think we can move it to Suppl
%For our semi‐supervised experiments, we sample $10\%$ of the Cityscapes training set as labeled and use the remaining images as unlabeled. For ADE20K, to reduce computational cost we use stratified sampling to select $20\%$ of the training pool while preserving the original per category distribution, yielding 3,537 images. From this subset, 1,000 images (10 per instance class) serve as labeled data and the remaining 2,537 as unlabeled. Detailed statistics are provided in Appendix~A, and we will release the exact split files alongside our code.  

\noindent\textbf{Teacher and Student Architectures.}
Our teacher is a fused ensemble of Grounding-DINO-Large~\cite{liu2024grounding} and SAM2-L~\cite{ravi2024sam}. Since the latest Grounding-DINO model is not fully open-source, we use its open-source counterpart, mm-Grounding-DINO~\cite{zhao2024open}. For the student, we use a compact architecture consisting of a DINOv2-S encoder~\cite{oquab2023dinov2}, a DPT-S decoder head~\cite{ranftl2021vision}, and a lightweight transformer decoder in the style of Mask2Former~\cite{cheng2022masked}. This design provides a strong trade-off between accuracy and efficiency while remaining substantially smaller and more deployable than the teacher. We analyze alternative student designs in Sec.~\ref{encoder-decoder}, and provide optimization details and hyperparameters in Appendix~B.

\subsection{Main Results}
\label{sec:results}
We compare our method against a range of baselines, including supervised fine-tuning and recent semi-supervised knowledge distillation approaches. Table~\ref{tab:results_both} reports maskAP and maskAP${50}$ on Cityscapes and ADE20K. In the teacher adaptation stage (568M parameters), adding our pixel-level contrastive loss improves performance over self-training on both benchmarks. On Cityscapes, maskAP increases from 29.8 to 30.5 (+0.7) and maskAP${50}$ from 54.9 to 56.6 (+1.7). On ADE20K, maskAP improves from 14.8 to 15.2 (+0.4) and maskAP$_{50}$ from 23.7 to 24.5 (+0.8). In the teacher adaptation setting, the absolute gains are modest because the teacher already starts from a strong pretrained foundation model. The consistent improvements across datasets nevertheless suggest that pixel-wise contrastive regularization improves feature discrimination and yields more spatially consistent pseudo-labels for downstream student distillation.

In the student distillation stage, our 52M-parameter student (about 9\% of the composite teacher size) achieves 32.2 maskAP and 56.5 maskAP$_{50}$ on Cityscapes, outperforming prior semi-supervised distillation baselines. After the final refinement stage, the student reaches 33.9 maskAP and 58.7 maskAP$_{50}$, surpassing the adapted teacher by +3.4 maskAP. On ADE20K, the student attains 16.1 maskAP and 27.4 maskAP$_{50}$ after knowledge transfer, and further improves to 16.7 maskAP and 28.0 maskAP$_{50}$ after refinement, again exceeding the adapted teacher. These results show that the proposed pipeline transfers knowledge effectively from a large foundation model into a substantially smaller student across both benchmarks. Additional ablations under varied label splits are presented in Sec.~\ref{split}. To compare efficiency, Fig.~\ref{fig:scatter_chart} summarizes key efficiency metrics for the teacher and student models on a logarithmic scale.

\begin{table*}[ht]
  \centering
  \caption{\textbf{Main results on Cityscapes and ADE20K} with 10\% labeled data.
We report teacher adaptation (568M) and student distillation (52M).
* denotes adapted methods. Blue shading indicates the best and second-best results.}
  \label{tab:results_both}
  \setlength\tabcolsep{6pt}
  \begin{adjustbox}{width=\textwidth}
  \begin{tabular}{l l S[table-format=2.1] S[table-format=2.1] S[table-format=2.1] S[table-format=2.1]}
    \toprule
    \textbf{Method} & \textbf{Data Regime}
      & \multicolumn{2}{c}{\textbf{Cityscapes}}
      & \multicolumn{2}{c}{\textbf{ADE20K}} \\
    \cmidrule(lr){3-4}\cmidrule(lr){5-6}
    & & {maskAP} & {maskAP$_{50}$} & {maskAP} & {maskAP$_{50}$} \\
    \midrule[0.75pt]

    \multicolumn{6}{l}{\textit{Teacher Adaptation}} \\
    \cmidrule{1-6}
    Zero-shot VFM & None (pretrained) & 22.0 & 42.3 & 8.1 & 18.2 \\
    Supervised fine-tuning & Labeled only & 28.7 & 53.4 & 14.2 & 23.5 \\
    Self-training*~\cite{xie2020self} & Labeled+Unlabeled & 29.7 & 54.9 & 14.6 & 23.6 \\
    Unbiased Teacher*~\cite{liu2021unbiased} & Labeled+Unlabeled & \cellcolor{lightblue1}29.8 & \cellcolor{lightblue1}54.9 & \cellcolor{lightblue1}14.8 & \cellcolor{lightblue1}23.7 \\
    ours & Labeled+Unlabeled & \cellcolor{lightblue3}30.5 & \cellcolor{lightblue3}56.6 & \cellcolor{lightblue3}15.2 & \cellcolor{lightblue3}24.5 \\
    \midrule[0.75pt]

    \multicolumn{6}{l}{\textit{Student Distillation}} \\
    \cmidrule{1-6}
    Supervised fine-tuning & Labeled only & 21.1 & 38.7 & 13.9 & 24.2 \\
    PAIS~\cite{hu2023pseudo} & Labeled+Unlabeled & 22.9 & 44.9 & 10.3 & 18.3 \\
    Guided dist.~\cite{berrada2024guided} & Labeled+Unlabeled & 30.8 & 52.9 & 14.2 & 23.8 \\
    Vemulapalli et al.*~\cite{Vemulapalli2024KD} & Unlabeled only & 24.4 & 45.6 & 5.1 & 9.3 \\
    Depth-Guided~\cite{chen2024depth} & Labeled+Unlabeled & 30.9 & 52.9 & \text{-} & \text{-} \\
    $S^{4}M$~\cite{yoon2025s} & Labeled+Unlabeled & \cellcolor{lightblue1}33.3 & \cellcolor{lightblue1}56.7 & \text{-} & \text{-} \\
    ours (knowledge transfer) & Labeled+Unlabeled & 32.2 & 56.5 & \cellcolor{lightblue1}16.1 & \cellcolor{lightblue1}27.4 \\
    ours (student refinement) & Labeled only & \cellcolor{lightblue3}33.9 & \cellcolor{lightblue3}58.7 & \cellcolor{lightblue3}16.7 & \cellcolor{lightblue3}28.0 \\
    \bottomrule
  \end{tabular}
  \end{adjustbox}
\end{table*}

%\setlength{\intextsep}{0.5\baselineskip}   % tighten space above/below wrap
%\begin{wrapfigure}[13]{r}[0pt]{0.5\textwidth}
%  \centering
%  \vspace{-10pt}
%  \includegraphics[width=\linewidth]{figures/scatter_log.pdf}
%  \vspace{-3mm}
%  \caption{Efficiency comparison (log scale).
%  }
%  \label{fig:scatter_chart}
%\end{wrapfigure}

\begin{figure}
    \centering
    \includegraphics[width=0.85\linewidth]{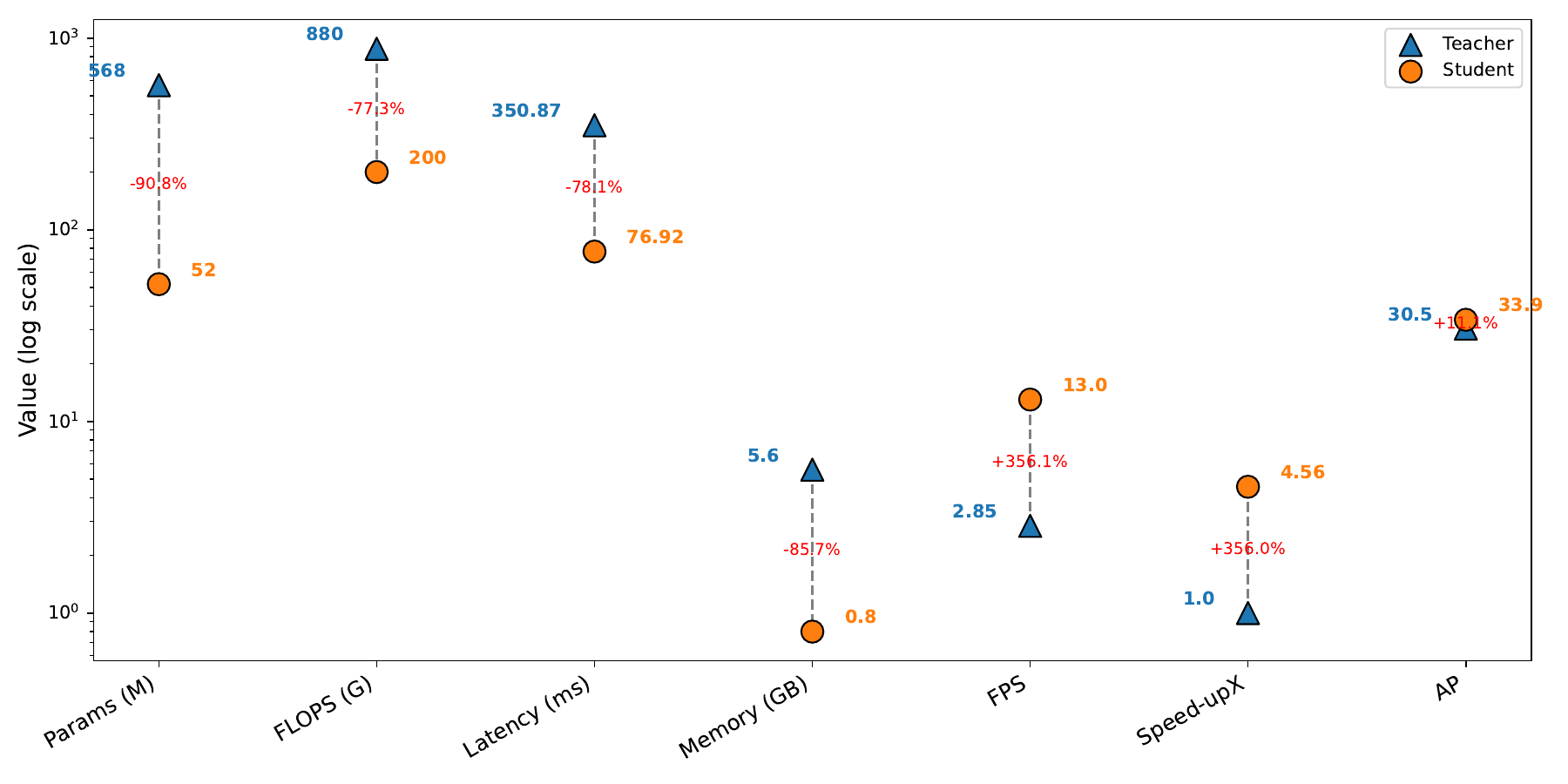}
    \caption{Efficiency comparison (log scale).}
    \label{fig:scatter_chart}
\end{figure}

\subsection{Empirical Validation}
\label{emp}
We empirically validate Proposition~\ref{prop:margin-growth} by monitoring the false negative rate (\(\mathrm{FNR}\)), defined as the fraction of sampled negatives that actually belong to the same instance, together with the empirical margin
\[
  \Delta_{\rm emp} = \mathrm{NegMean} - \mathrm{PosMean}.
\]
Let \(p = 1 - \mathrm{FNR}\) denote the probability of sampling a true negative. Figure~\ref{fig:metrics} reports the empirical margin every 10k iterations for \(\lambda_{\rm pxl}\in\{0.01,0.05,0.1,0.2\}\) (left), the false negative rate for \(\lambda_{\rm pxl}=0.1\) (center, dashed at \(p=0.5\)), and the contrastive loss for \(\lambda_{\rm pxl}=0.1\) (right). Throughout training, we observe \(p>0.9\) and an approximately linear increase in \(\Delta_{\rm emp}\) with \(\lambda_{\rm pxl}\), consistent with Proposition~\ref{prop:margin-growth}.

\begin{figure}[ht]
  \centering
  \begin{subfigure}[b]{0.31\columnwidth}
    \includegraphics[width=\textwidth]{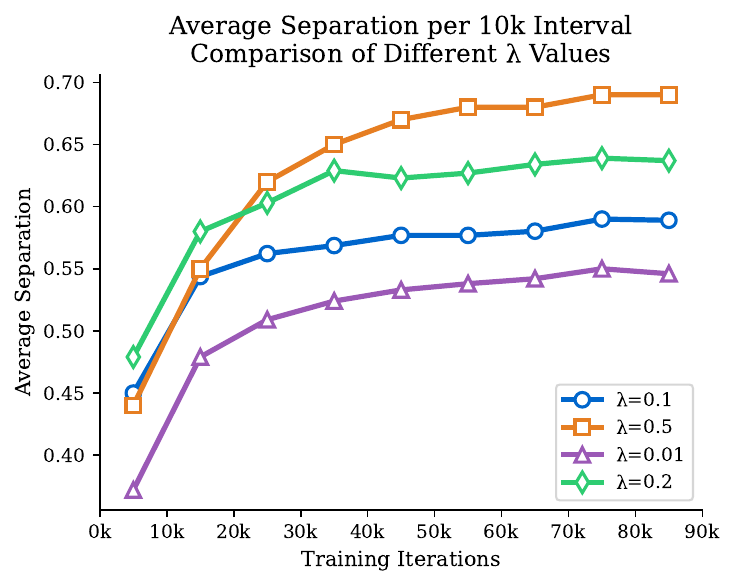}
    \caption{Empirical Margin}
    \label{fig:sep}
  \end{subfigure}\hfill
  \begin{subfigure}[b]{0.31\columnwidth}
    \includegraphics[width=\textwidth]{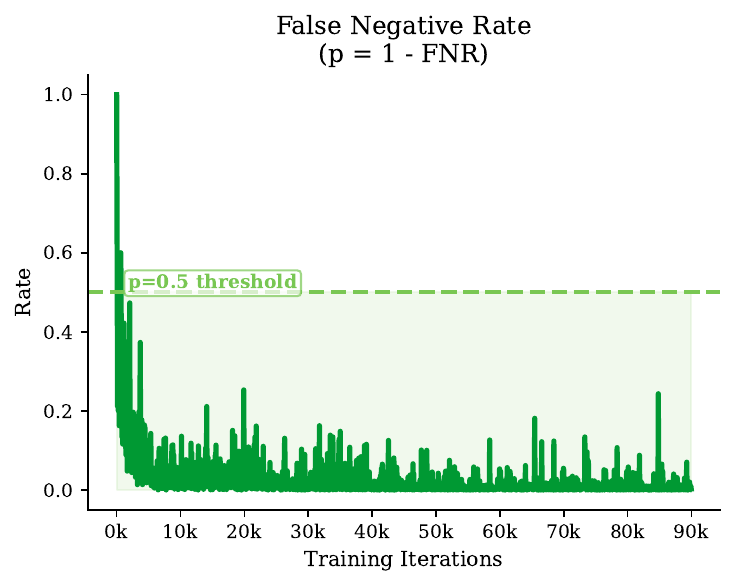}
    \caption{\(\mathrm{FNR}\)}
    \label{fig:fnr}
  \end{subfigure}\hfill
  \begin{subfigure}[b]{0.31\columnwidth}
    \includegraphics[width=\textwidth]{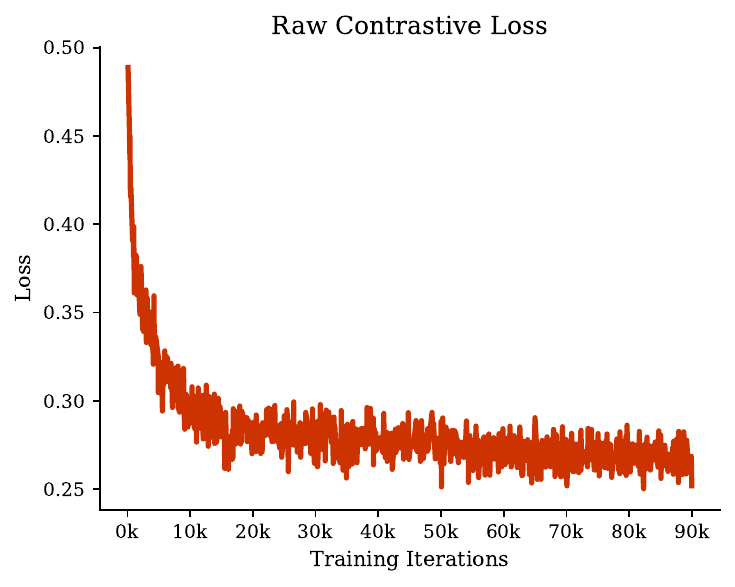}
    \caption{Contrastive Loss}
    \label{fig:loss}
  \end{subfigure}
  \caption{
    \textbf{Left:} Empirical margin (\(\mathrm{NegMean}-\mathrm{PosMean}\)) measured every 10k iterations for different values of \(\lambda_{\rm pxl}\).
    \textbf{Center:} False negative rate (\(\mathrm{FNR}\)) for \(\lambda_{\rm pxl}=0.1\), with the dashed line marking \(p=0.5\).
    \textbf{Right:} Contrastive loss for \(\lambda_{\rm pxl}=0.1\).
  }
  \label{fig:metrics}
\end{figure}

\subsection{Ablation Studies}
We perform ablation experiments to isolate the contribution of each component in the proposed pipeline. Unless otherwise noted, all ablations are conducted on Cityscapes with 10\% labeled data.
%%These include analyses of loss functions, training stages, negative sampling strategies, hyperparameters, and student architecture choices.  

\noindent\textbf{Impact of Loss Components.}
The distillation objective combines three terms: supervised loss (\(\mathcal{L}_{\rm sup}\)), pseudo-label loss (\(\mathcal{L}_{\rm semi}\)), and pixel-level contrastive loss (\(\mathcal{L}_{\rm pxl}\)). Table~\ref{tab:loss_ablation} shows that adding \(\mathcal{L}_{\rm semi}\) improves student performance from 21.1 to 30.7 maskAP, while further including \(\mathcal{L}_{\rm pxl}\) yields the best result of 32.2 maskAP. These results indicate that pseudo-label supervision and pixel-level contrastive regularization provide complementary benefits.

% ---------- Table 2 ----------
\begin{table}[htbp]
  \centering
  \caption{Ablations on Cityscapes (10\% labels): effect of loss terms.}
  \resizebox{0.9\linewidth}{!}{
  \begin{tabular}{lccc cc}
    \toprule
    \textbf{Method} & $\mathcal{L}_{\rm sup}$ & $\mathcal{L}_{\rm semi}$ & $\mathcal{L}_{\rm pxl}$ & Teacher & Student \\
    \midrule
    (a) Sup. only     & \checkmark &   &   & 28.7 & 21.1 \\
    (b) + Pseudo      & \checkmark & \checkmark &   & 29.7 & 30.7 \\
    (c) + Pixel loss  & \checkmark &   & \checkmark & 29.6 & 27.5 \\
    (d) (b)+(c)       & \checkmark & \checkmark & \checkmark & \bf 30.5 & \bf 32.2 \\
    \bottomrule
  \end{tabular}}
  \label{tab:loss_ablation}
\end{table}

\noindent\textbf{Impact of Training Stages.}
Beyond individual loss terms, we further ablate the contribution of each training stage. Table~\ref{tab:stage_ablation} shows the effect of removing one stage at a time. The supervised baseline achieves 21.1 maskAP. Adding distillation alone improves performance to 23.8 (+2.7), and further adding student fine-tuning raises it to 32.2 (+8.4). Without teacher adaptation, performance drops to 25.7, highlighting the importance of aligning the teacher with the target domain. The full three-stage pipeline achieves the best result of 33.9 maskAP, a gain of +12.8 over the supervised baseline.

% ---------- Table 3 ----------
\begin{table}[htbp]
  \centering
  \caption{Ablations on Cityscapes (10\% labels): effect of training stages.}
  \resizebox{0.9\linewidth}{!}{
  \begin{tabular}{lccc c}
    \toprule
    \textbf{Variant} & Teacher Adapt. & Distill. & Student FT & maskAP \\
    \midrule
    Full pipeline          & \checkmark & \checkmark & \checkmark & 33.9 \\
    No Student FT          & \checkmark & \checkmark &            & 32.2 \\
    No Teacher Adapt.      &            & \checkmark & \checkmark & 25.7 \\
    Distillation Only      &            & \checkmark &            & 23.8 \\
    No Distill. (Sup.)     &            &            & \checkmark & 21.1 \\
    \bottomrule
  \end{tabular}}
  \label{tab:stage_ablation}
\end{table}

%\noindent\textbf{Impact of Training Stages.}
%Beyond the contribution of individual loss terms, we further ablate each stage of CAST to justify their necessity. 
%Table~\ref{tab:stage_ablation} shows results on Cityscapes (10\% labels), where we drop exactly one stage at a time.

%\noindent The supervised baseline achieves 21.1 maskAP. Adding distillation alone improves this to 23.8 (+2.7), and further adding student fine-tuning raises it to 32.2 (+8.4). 
%Without teacher adaptation, performance drops to 25.7, underscoring the need to align the teacher with the target domain. 
%The full three-stage CAST pipeline achieves best result of 33.9 maskAP, a +12.8 improvement over baseline.
\begin{figure*}[!htpb]
\begin{center}
   \includegraphics[width=1\linewidth]{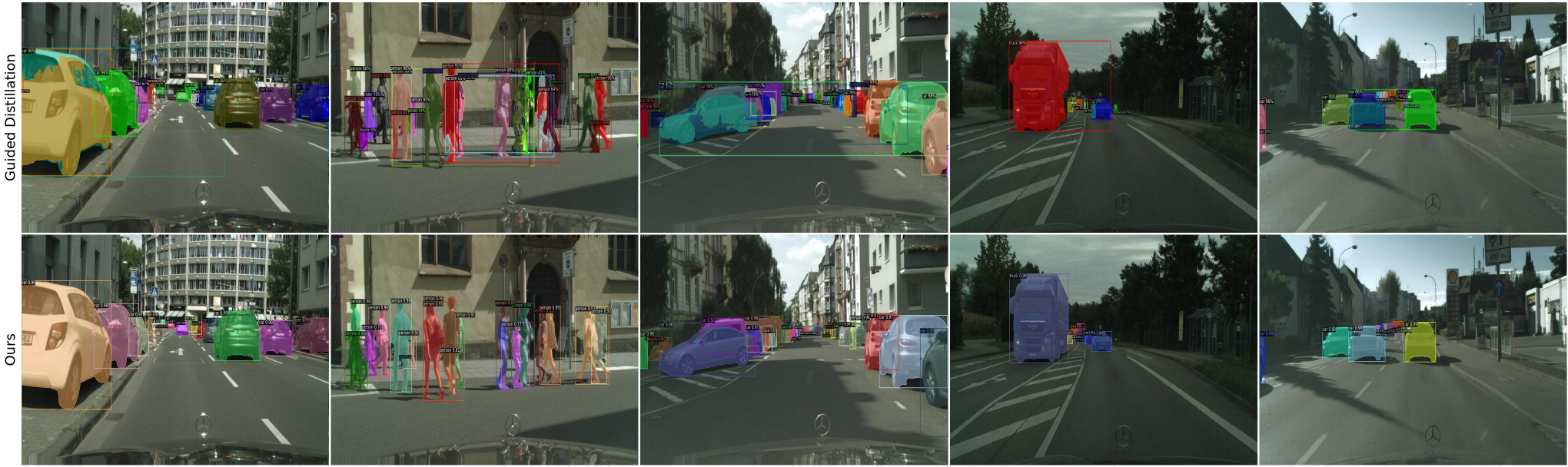}
\end{center}
\caption{\textbf{Qualitative results on Cityscapes.} Guided dist.~\cite{berrada2024guided} (top) versus our method (bottom).}
\label{fig:qual}
\end{figure*}
\noindent\textbf{Ablation of Negative Sampling Strategies.}
To validate the proposed negative sampling strategy in the pixel-level contrastive loss, Table~\ref{tab:sampling_ablation} compares four variants: \emph{Uniform}, which samples negatives uniformly across the image; \emph{Mask-Only}, which derives the sampling distribution solely from mask predictions; \emph{Class-Only}, which uses only class predictions; and \emph{Fusion}, which combines mask and class predictions. The fusion strategy achieves the best performance, reaching 32.2 maskAP and 56.5 maskAP$_{50}$, indicating that the two sources of information are complementary for identifying informative negatives.

%\noindent\textbf{Ablation of Negative Sampling via Various Probability Maps.}
%To validate our negative sampling strategy in the pixel-level contrastive loss, Table~\ref{tab:mask_class_fusion} compares four sampling methods: \textbf{Uniform:} negatives sampled uniformly across the image; \textbf{Mask-Only:} The probability map is derived solely from mask predictions, with class probabilities assumed to be uniform. \textbf{Class-Only:} The map is generated only from class predictions, assuming a uniform spatial distribution for the mask.  \textbf{Fusion:} Combining both mask and class predictions. The fusion strategy achieves the best results, with 32.2 maskAP and 56.5 AP$_{50}$.

\begin{table}[htbp]
  \centering
  \caption{
    \textbf{Ablation of negative sampling strategies on Cityscapes.}
    Top: quantitative results for uniform, mask-only, class-only, and fusion samplers.
    Bottom: schematic illustration of the corresponding pixel-level sampling distributions.
  }
  \label{tab:sampling_ablation}
  \vspace{0.3em}
  \resizebox{0.8\linewidth}{!}{
  \begin{tabular}{lSS}
    \toprule
    Method      & {maskAP (\%)} & {maskAP$_{50}$ (\%)} \\
    \midrule
    Uniform      & 29.4 & 50.2 \\
    Mask-Only    & 30.6 & 55.0 \\
    Class-Only   & 31.1 & 55.3 \\
    Fusion       & \bfseries 32.2 & \bfseries 56.5 \\
    \bottomrule
  \end{tabular}}
  \vspace{0.8em}

  \includegraphics[width=0.7\linewidth, height=3cm, keepaspectratio=false]{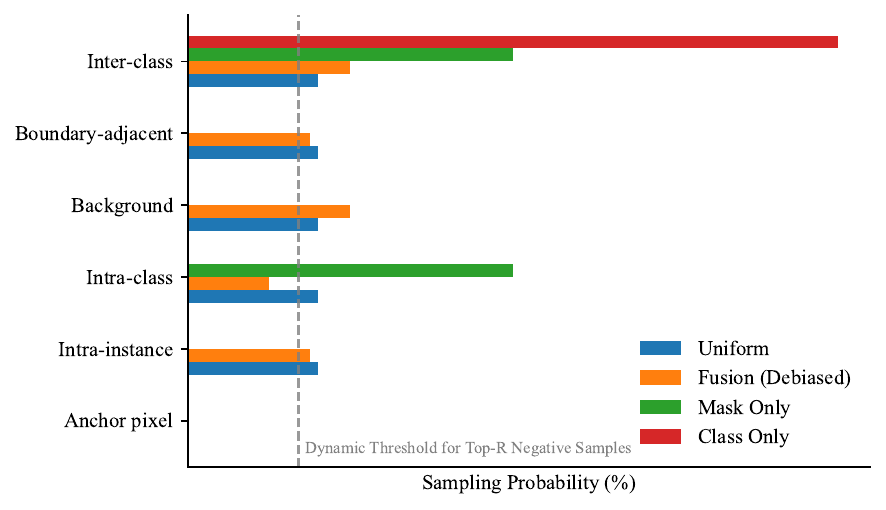}
\end{table}

\noindent\textbf{Student Architecture Variants.}
\label{encoder-decoder}
We evaluate two design axes for the student model under the distillation protocol: (i) the encoder backbone with a fixed DPT decoder, and (ii) the decoder head with a fixed DINOv2-S encoder. Table~\ref{tab:arch_ablation} reports accuracy and parameter counts on the Cityscapes validation set. The combination of a DINOv2-S encoder and DPT head provides the best accuracy while maintaining a compact model size.

\begin{table}[!htbp]
  \centering
  \small
  \caption{
    \textbf{Architecture ablations on Cityscapes.}
    Top: encoder backbone comparison with a fixed DPT decoder.
    Bottom: decoder head comparison with a fixed DINOv2-S encoder.
  }
  \label{tab:arch_ablation}

  \vspace{0.3em}
  \renewcommand{\arraystretch}{0.7}
  \resizebox{0.9\linewidth}{!}{
  \begin{tabular}{lccc}
    \toprule
    Encoder    & maskAP & maskAP$_{50}$ & Params (M) \\
    \midrule
    ResNet50   & 25.5 & 49.3 & 24 \\
    SAM2-S     & 22.1 & 39.2 & 35 \\
    DINOv2-S   & \bfseries 30.7 & \bfseries 54.9 & 22 \\
    \bottomrule
  \end{tabular}}
  \vspace{0.8em}

  \resizebox{0.9\linewidth}{!}{
  \begin{tabular}{lccc}
    \toprule
    Decoder    & maskAP & maskAP$_{50}$ & Params (M) \\
    \midrule
    FPN        & 28.9 & 52.4 & 18 \\
    DPT        & \bfseries 30.7 & \bfseries 54.9 & 22 \\
    \bottomrule
  \end{tabular}}
\end{table}

\begin{figure*}[!htpb]
\begin{center}
   \includegraphics[width=1\linewidth]{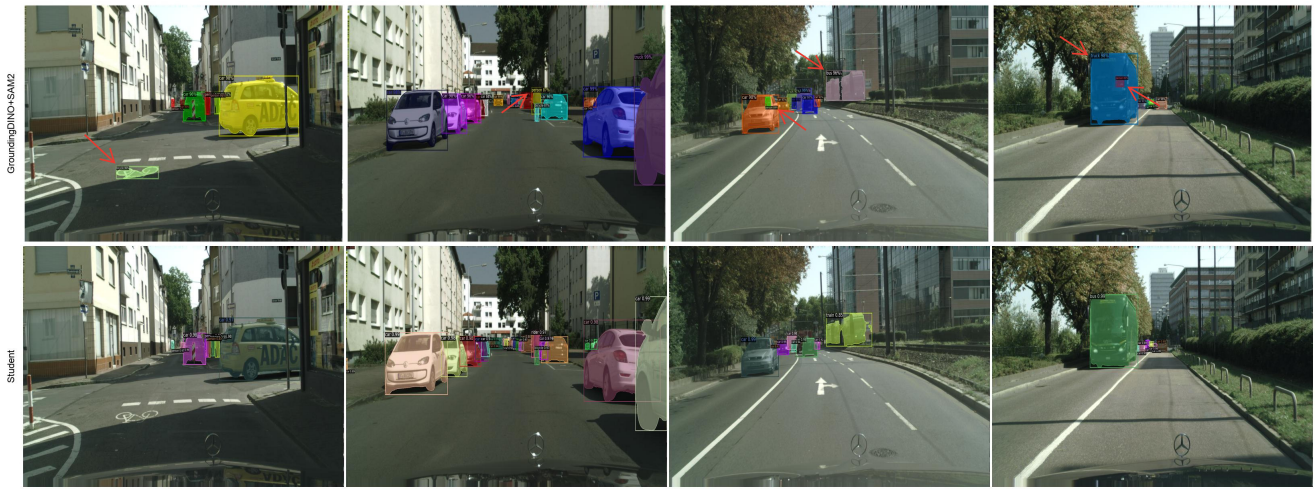}
   \end{center}
\caption{\textbf{Qualitative bias reduction in stage-wise distillation.}
Top row: pseudo-labels generated by the adapted teacher.
Bottom row: student predictions after distillation and refinement, showing reduced pseudo-label bias and sharper instance boundaries.}
\label{fig:qual3}
\end{figure*}

\begin{figure*}[!htpb]
\begin{center}
   \includegraphics[width=0.96\linewidth]{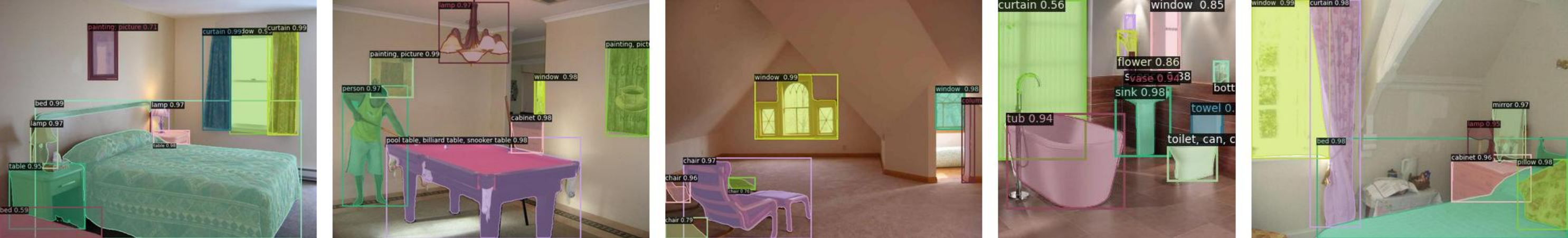}
   \end{center}
   \caption{\textbf{Qualitative results on ADE20K.}}
\label{fig:qual2}
\end{figure*}
\begin{figure}[!htpb]
    \centering
    \includegraphics[width=0.85\linewidth]{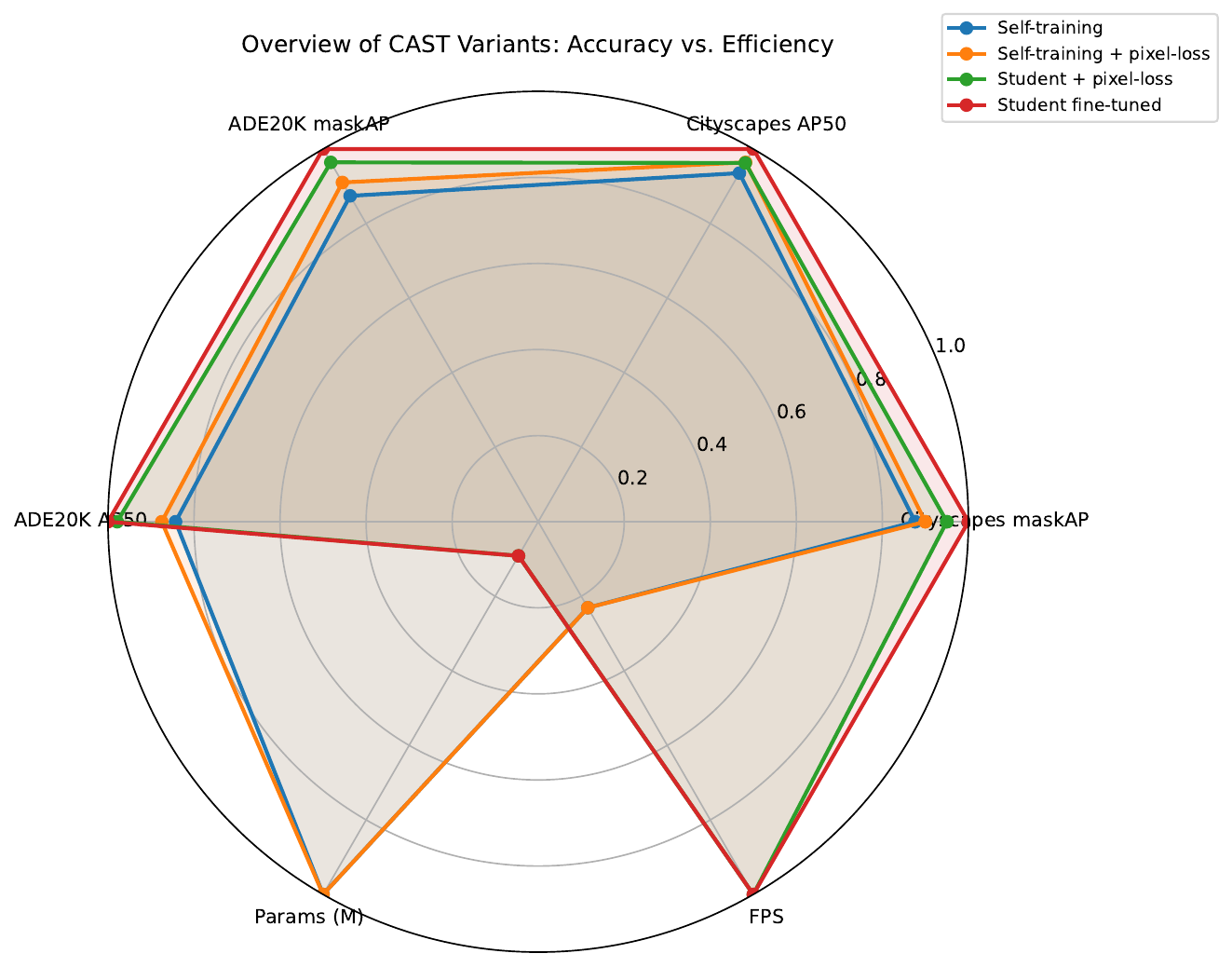}
    \caption{Performance--complexity radar chart (normalized).}
    \label{fig:radar_chart}
\end{figure}

\noindent\textbf{Scalability with Labeled Fractions.}
\label{split}
We evaluate the method under different fractions of labeled data to assess robustness in low-label regimes. Following the protocol in~\cite{berrada2024guided}, we train with 5\%, 10\%, and 30\% labeled splits of Cityscapes. As shown in Table~\ref{tab:scalability}, the method consistently outperforms prior approaches across all fractions. At 5\% labels, it achieves 30.7 AP, substantially exceeding PAIS (18.0) and Guided Distillation (23.0). At 30\% labels, it reaches 40.4 AP, surpassing the strongest baseline (37.8 from S$^4$M) by +2.6 AP. These results show that the method remains effective under scarce supervision while scaling favorably with additional labeled data. \noindent Additional ablations are provided in the supplementary material.
%including teacher adaptation variants, alternative loss formulations, sampling scope, and backbone comparisons, 

\begin{table}[htbp]
  \centering
  \caption{\textbf{Scalability across label fractions on Cityscapes.} Results with different proportions of labeled data.}
  \label{tab:scalability}
  \resizebox{0.95\linewidth}{!}{
  \begin{tabular}{llcccccc}
    \toprule
    \textbf{Dataset} & \textbf{Fraction} & Teacher Adapt. & Distillation & Student (refined) & PAIS~\cite{hu2023pseudo} & Guided dist.~\cite{berrada2024guided} & S$^4$M~\cite{yoon2025s} \\
    \midrule
    \multirow{3}{*}{Cityscapes} 
      & 5\%   & 29.4 & 29.2 & \textbf{30.7} & 18.0 & 23.0 & 30.1 \\
      & 10\%  & 30.5 & 32.2 & \textbf{33.9} & 22.9 & 30.8 & 33.3 \\
      & 30\%  & 33.3 & 38.5 & \textbf{40.4} & 32.8 & 35.6 & 37.8 \\
    \bottomrule
  \end{tabular}
  }
\end{table}

%% file: sec/5_conclusion.tex
\section{Conclusion}
We presented a semi-supervised knowledge distillation pipeline that combines self-training, instance-aware pixel-wise contrastive learning, and final supervised refinement to transfer knowledge from large vision foundation models into compact student experts. Empirically, the resulting student is approximately \(11\times\) smaller than the teacher while surpassing the adapted teacher by +3.4 maskAP on Cityscapes and +1.5 maskAP on ADE20K. These results show that pixel-level contrastive regularization can improve pseudo-label quality and enable efficient low-label adaptation of strong foundation models. Our theoretical analysis further supports the proposed negative sampling strategy by showing that, under mild assumptions, it increases the expected inter-instance margin. Future work includes simplifying the multi-stage pipeline, evaluating on additional domains, and extending the framework to broader efficient perception settings.
\section{Acknowledgments}
Portions of this research were conducted with the advanced computing resources provided by Texas A\&M High Performance Research Computing.

%% file: sec/X_suppl.tex
\clearpage
\setcounter{page}{1}
%\maketitlesupplementary

%% arXiv does not compile this file standalone

\section*{Supplementary Material}
\label{supplementary}

This document provides additional details to support the main paper, including dataset statistics, full hyperparameter settings, formal proof, extended training protocols, and additional ablation studies. 

\begin{comment}
    
\section{Dataset Splits}
\label{sec:datasets}

Table~\ref{tab:datasets} summarizes the datasets used in our experiments. We use a 10\% labelled split of Cityscapes’ 2\,975 training images (298 labeled / 2\,677 unlabeled) and a stratified 20\% split of ADE20K’s 20\,210 training images (1\,000 labeled / 2\,537 unlabeled). Standard validation sets are retained (500 images for Cityscapes, 2\,000 for ADE20K). Exact image‐ID lists will be released with our code.

\begin{table}[ht]
  \centering
  \small
  \caption{Semi-supervised splits used in our experiments.}
  \label{tab:datasets}
  \begin{tabular}{lccc}
    \toprule
    Dataset          & \# Classes & Labeled / Unlabeled & Validation \\
    \midrule
    Cityscapes       & 8          & 298 / 2\,677        & 500        \\
    ADE20K           & 100        & 1\,000 / 2\,537     & 2\,000      \\
    \bottomrule
  \end{tabular}
\end{table}
\end{comment}

\section{Hyperparameters}
\label{sec:hyperparams}

Key teacher and student hyperparameters are summarized in Table~\ref{tab:hyperparams}. Results are averages over three independent runs with different random seeds.

\begin{table*}[ht]
  \centering
  \small
  \caption{Hyperparameter settings.}
  \label{tab:hyperparams}
  \begin{tabular}{lcc}
    \toprule
    \textbf{Parameter}                  & \textbf{Teacher}                       & \textbf{Student}                                          \\
    \midrule
    Learning rate                       & $5.0\times10^{-5}$                     & Encoder: $5.0\times10^{-6}$; Decoder: $5.0\times10^{-5}$ \\
    Scheduler                           & Multi-step (milestones at 0.9, 0.95)   & PolyLR (power 0.9)                                        \\
    Batch size                          & 4                                      & 8                                                         \\
    \midrule
    Weight decay   & 0.01 & 0.05  \\
    Contrastive loss weight    & 0.2   & 0.2    \\
    Pseudo-label threshold    & 0.3   & 0.3    \\
    Dropout rate         & —       & 0.1   \\
    Gradient clipping    & —    & $\ell_2$ norm 0.1   \\
    \midrule
    Optimizer    & \multicolumn{2}{c}{AdamW ($\beta_1=0.9$, $\beta_2=0.999$)}      \\
    Augmentations     & \multicolumn{2}{c}{Weak: flip, resize; Strong: random resized crop, jitter, grayscale, blur } \\
    Loss weights (mask / class)         & \multicolumn{2}{c}{5 / 2}                                                \\
    \bottomrule
  \end{tabular}
\end{table*}

\section{Proof Sketch of Proposition~3.1}
\label{sec:prof-sketch}

\begin{proof}[Proof Sketch]
Let \(z_a\), \(z^+\) and \(\{z^-_r\}_{r=1}^R\) be the unit norm embeddings of an anchor pixel, its positive, and \(R\) negatives.  Define
\[
s^+ = \langle z_a,\,z^+\rangle,
\qquad
s^-_r = \langle z_a,\,z^-_r\rangle,
\]
and the pixel-wise contrastive loss
\[
\ell(z_a)
= -\log\frac{\exp(s^+)}{\exp(s^+)+\sum_{r=1}^R\exp(s^-_r)}.
\]
Let
\[
Z = \exp(s^+)+\sum_{r=1}^R\exp(s^-_r),
\qquad
\alpha_r = \frac{\exp(s^-_r)}{Z}.
\]
A straightforward gradient computation gives
\[
\nabla_{z_a}\ell
= \sum_{r=1}^R \alpha_r\,(z^-_r - z^+).
\]
Applying one gradient descent step with step size \(\lambda_{\rm pxl}\):
\[
z_a' = z_a - \lambda_{\rm pxl}\,\nabla_{z_a}\ell
\;=\;
z_a + \lambda_{\rm pxl}\sum_{r=1}^R \alpha_r\,(z^+ - z^-_r).
\]
For a randomly chosen negative \(z^-\),
\begin{align*}
\Delta s^+
&= \langle z_a' - z_a,\,z^+\rangle
= \lambda_{\rm pxl}\sum_{r=1}^R \alpha_r\bigl(1 - \langle z^-_r,\,z^+\rangle\bigr),\\
\Delta s^-
&= \langle z_a' - z_a,\,z^-\rangle
= \lambda_{\rm pxl}\sum_{r=1}^R \alpha_r\bigl(\langle z^+,\,z^-\rangle - \langle z^-_r,\,z^-\rangle\bigr).
\end{align*}

By Assumption~3.1, each negative embedding \(z^-_r\) is inter-instance with probability \(p\), in which case
\(\langle z^-_r,\,z^+\rangle\approx0\), and intra-instance with probability \(1-p\), in which case
\(\langle z^-_r,\,z^+\rangle\approx1\).  Hence
\[
  \EE\bigl[1 - \langle z^-_r,\,z^+\rangle\bigr]
  = p\cdot1 + (1-p)\cdot0
  = p,
\]
and since \(\sum_{r=1}^R\alpha_r=1\), it follows that
\[
  \EE[\Delta s^+]
  = \lambda_{\rm pxl}\sum_{r=1}^R\alpha_r\,
    \EE\bigl[1 - \langle z^-_r,\,z^+\rangle\bigr]
  = p\,\lambda_{\rm pxl}.
\]
Meanwhile, every term in \(\Delta s^-\) involves an inter-instance inner product, either
\(\langle z^+,z^-\rangle\) or \(\langle z^-_r,z^-\rangle\) each of which vanishes in expectation, so
\(\EE[\Delta s^-]\approx0\).  Therefore
\[
  \EE[\Delta s^+ - \Delta s^-]
  = p\,\lambda_{\rm pxl} - 0
  = \Theta\!\bigl(p\,\lambda_{\rm pxl}\bigr)
  = \varepsilon>0,
\]
i.e.\ one update on \(\mathcal L_{\rm pxl}\) increases the expected inter-instance margin by \(\varepsilon\).
\end{proof}

\begin{remark}[Why \(\langle z^+, z^-\rangle \approx 0\) holds]
Under the InfoNCE objective (§3.2), the normalized weights for negative pairs,
\(\alpha_r = \frac{e^{s^-_r}}{e^{s^+} + \sum_{r}e^{s^-_r}},\)
vanish at convergence, i.e.\ \(\alpha_r \approx 0\).  Moreover, in high dimensional embeddings, random unit vectors have inner products concentrating near zero, and contrastive training further pushes these negative similarities into a tight, small magnitude distribution~\cite{chen2020simple}.  Thus it is reasonable to approximate \(\langle z^+,z^-\rangle\approx0\) up to \(O(1/\sqrt{D})\) fluctuations.
\end{remark}

\section{More Training Details}
\label{sec:training-details}
All teacher models are first fine-tuned for 1k iterations on the labeled set, followed by 5k iterations of self-training with pseudo-labels. For student models, training on the Cityscapes dataset spans 90k iterations, consistent with prior works, while ADE20k dataset is trained for 80k iterations. Finally, both datasets undergo an additional supervised fine-tuning phase for 2k iterations.

\section{Additional Ablation Studies}
\label{sec:more-ab}

\subsection{Hyperparameter Sensitivity}
\label{sec:ablate-hyperparam}

We analyze the sensitivity of the proposed contrastive learning component to three key hyperparameters on Cityscapes: the contrastive loss weight $\lambda_{\rm pxl}$, the number of negatives per anchor $K$, and the temperature $T$. Table~\ref{tab:hyperparam} reports teacher and student performance (maskAP and maskAP$_{50}$) across the full hyperparameter sweep.

For the contrastive weight, performance improves steadily as $\lambda_{\rm pxl}$ increases from $0$ to $0.2$, indicating that pixel-level contrastive supervision provides useful regularization during both teacher adaptation and student distillation. Larger values (e.g., $0.5$) begin to slightly degrade performance, suggesting that overly strong contrastive weighting may interfere with the supervised objective.

For the number of negatives, $K=256$ provides the best trade-off between accuracy and computational cost. Although $K=512$ slightly improves teacher maskAP, the gains are marginal relative to the increased sampling overhead. 

Finally, the temperature parameter shows stable behavior, with $T=0.2$ consistently yielding the best performance across both teacher and student models. Based on these results, we adopt $\lambda_{\rm pxl}=0.2$, $K=256$, and $T=0.2$ in all experiments reported in the main paper.

\begin{table}[!htbp]
  \centering
  \small
  \caption{\textbf{Hyperparameter ablation on Cityscapes.}}
  \label{tab:hyperparam}
  \resizebox{\linewidth}{!}{%
  \begin{tabular}{ll
                  SSSSS
                  SSS
                  SSS}
    \toprule
    \multirow{2}{*}{Model} 
      & \multirow{2}{*}{Metric}
      & \multicolumn{5}{c}{Contrastive Loss Weight ($\lambda_{\rm pxl}$)}
      & \multicolumn{3}{c}{Negatives per Anchor ($K$)}
      & \multicolumn{3}{c}{Temperature ($T$)} \\
    \cmidrule(lr){3-7} \cmidrule(lr){8-10} \cmidrule(lr){11-13}
      &
      & {0} & {0.01} & {0.1} & {\bfseries 0.2} & {0.5}
      & {128} & {\bfseries 256} & {512}
      & {0.1} & {\bfseries 0.2} & {0.4} \\
    \midrule
    \multirow{2}{*}{Teacher}
      & AP
              & 29.7 & 29.9 & 30.2 & \bfseries 30.5 & 30.1
              & 30.4 & \bfseries 30.5 & 30.9
              & 30.1 & \bfseries 30.5 & 29.8 \\
      & AP$_{50}$
              & 55.3 & 55.7 & 56.1 & \bfseries 56.6 & 56.1
              & 56.3 & \bfseries 56.6 & 57.1
              & 55.9 & \bfseries 56.6 & 55.3 \\
    \addlinespace
    \multirow{2}{*}{Student}
      & AP
              & 30.7 & 30.8 & 32.1 & \bfseries 32.2 & 30.9
              & 29.8 & \bfseries 32.2 & 32.1
              & 31.9 & \bfseries 32.2 & 31.7 \\
      & AP$_{50}$
              & 54.9 & 55.2 & 56.2 & \bfseries 56.5 & 55.7
              & 55.3 & \bfseries 56.5 & 56.6
              & 56.0 & \bfseries 56.5 & 55.8 \\
    \bottomrule
  \end{tabular}
  }
\end{table}

\subsection{Ablation: Teacher Adaptation Variants}
\label{sec:ablation-teacher}

Different teacher adaptation strategies impact both teacher and student performance. Specifically, we compare fine-tuning only, self-training, and self-training combined with our proposed 
contrastive loss.

\begin{table}[ht]
  \centering
  \small
  \caption{\textbf{Teacher adaptation ablation.} Teacher and student AP under different adaptation strategies.}
  \label{tab:ablation_teacher}
  \setlength{\tabcolsep}{4pt}
  \begin{tabular}{lccc}
    \toprule
    \textbf{Variant} & \makecell{\textbf{Teacher}\\\textbf{AP}} & \makecell{\textbf{Student}\\\textbf{AP}} & \makecell{\textbf{$\Delta$ vs.}\\\textbf{SOTA}} \\
    \midrule
    Fine-tuning only         & 28.7 & 32.0 & +1.2 \\
    Self-training            & 29.7 & 32.2 & +1.5 \\
    Self-training + contr.   & \bfseries 30.5 & \bfseries 33.9 & \bfseries +3.1 \\
    \bottomrule
  \end{tabular}
\end{table}

\subsection{Loss Variant: InfoNCE vs.\ Margin Hinge}
\label{sec:ablation-loss-variant}

Replacing our asymmetric InfoNCE (§3.2) with a margin-based hinge loss yields identical maskAP (32.2\%) and +0.6 maskAP$_{50}$, at the cost of $1.6$× longer training. This evaluates whether enforcing a fixed positive–negative margin can match or improve upon the performance of InfoNCE.

\begin{table}[!htpb]
  \centering
  \small
  \caption{\textbf{Loss Variant Ablation.} Default InfoNCE vs.\ margin-based hinge (margin = 0.2).}
  \label{tab:ablation_loss_variant}
  \begin{tabular}{lcc}
    \toprule
    Loss Variant                     & maskAP (\%) & maskAP$_{50}$ (\%) \\
    \midrule
    Asymmetric InfoNCE (§3.2)        & 32.2        & 56.5           \\
    Margin hinge (m = 0.2)           & 32.2        & 57.1           \\
    \bottomrule
  \end{tabular}
\end{table}

\subsection{Ablation: Debias Score Formulation}
\label{sec:ablation-debias}
We evaluate three instantiations of the debias score function \(s^{\mathit{deb}}\) (§3.2):
\begin{itemize}
  \item \textbf{Original \(s^{\mathit{deb}}\):} fusion of mask and class confidences (ours).
  \item \(\bigl(s^{\mathit{deb}}\bigr)^2\): square each score to amplify the negatives with high confidence.
  \item \(\sqrt{s^{\mathit{deb}}}\): take the square root of each score to temper the bias.
\end{itemize}

% ---------- Top-down version ----------
\begin{table}[ht]
  \centering
  \small
  \caption{\textbf{Debias Score Formulation Ablation.}}
  \label{tab:ablation_debias}
  \begin{tabular}{lcc}
    \toprule
    \textbf{Score Variant} & \textbf{maskAP} & \textbf{maskAP$_{50}$} \\
    \midrule
    Original     & 32.2 & 56.5 \\
    Squared      & 32.0 & 56.3 \\
    Square‐root  & 31.9 & 56.2 \\
    \bottomrule
  \end{tabular}
\end{table}

\begin{table}[ht]
  \centering
  \small
  \caption{\textbf{Teacher Choice Ablation.}}
  \label{tab:ablation_teacher_choice}
  \begin{tabular}{lcc}
    \toprule
    \textbf{Model} & \textbf{AP} & \textbf{maskAP$_{50}$} \\
    \midrule
    Teacher T1 (0-shot) & 22.0 & 42.3 \\
    Teacher T2 (adapted) & 30.5 & 56.6 \\
    Student under T1 & 23.8 & 42.9 \\
    Student under T2 & \bfseries 32.2 & \bfseries 56.5 \\
    \bottomrule
  \end{tabular}
\end{table}

\subsection{Ablation: Negative Sampling Scope}
\label{sec:ablation-scope}
We evaluate two negative sampling scopes: (i) sampling only within the current mini batch vs.\ (ii) sampling from a small memory bank of past pixel embeddings (size 10k).

\begin{table}[htbp]
  \centering
  \small
  \caption{\textbf{Sampling scope ablation.} Mini-batch negatives vs.\ memory-bank negatives.}
  \label{tab:ablation_scope}
  \setlength{\tabcolsep}{4pt}
  \begin{tabular}{lcc}
    \toprule
    \textbf{Scope} & \makecell{\textbf{maskAP}\\\textbf{(\%)}} & \makecell{\textbf{maskAP$_{50}$}\\\textbf{(\%)}} \\
    \midrule
    Mini-batch only & 32.2 & 56.5 \\
    Memory bank (10k) & 32.7 & 57.3 \\
    \bottomrule
  \end{tabular}
\end{table}
Sampling from a memory bank of 10 k embeddings yields a modest performance gain (+0.5 maskAP, +0.8 maskAP$_{50}$) compared to in‐batch sampling. However, it incurs approximately $2.2$× longer training time due to the overhead of maintaining and querying the memory bank.

\subsection{Teacher Choice: Original vs.\ Adapted}
\label{sec:teacher-choice}

We compare distilling the student from the original VFM teacher (T1, zero-shot) versus our adapted teacher (T2). 
As shown in Table~\ref{tab:ablation_teacher_choice}, using the adapted teacher provides a much stronger signal, 
yielding a +8.4 AP improvement over the student distilled under T1. 

\begin{figure}
    \centering
\includegraphics[width=0.7\linewidth]{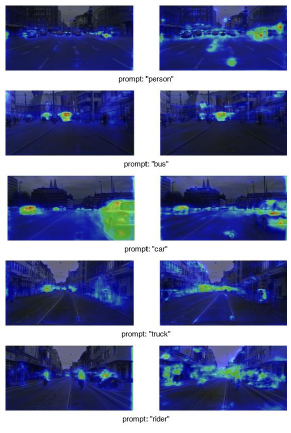}
   \caption{
Attention maps of the teacher model before and after adaptation.
Left: zero-shot VFM teacher before adaptation.
Right: teacher after adaptation with self-training and contrastive supervision.
The adapted teacher exhibits more localized attention on target objects
(person, bus, car, truck, rider) and reduced background activation,
indicating improved spatial discrimination that leads to higher-quality
pseudo-labels for student distillation.
}
    \label{fig:attn}
\end{figure}

\subsection{Extended Backbone Comparison}
\label{sec:backbone-comparison}

We compare the proposed DINOv2-S student against Guided Distillation baselines trained with different teacher backbones, including ResNet-50, DINOv2-B, and DINOv2-L.
\begin{table}[htpb]
\centering
\small
\caption{\textbf{Extended backbone comparison.} Proposed DINOv2-S student vs.\ Guided Distillation.}
\label{tab:backbone_comparison}
\begin{tabular}{lcccc}
\toprule
\textbf{Label} &
\textbf{Ours} &
\makecell{Guided Dist.\\ResNet-50} &
\makecell{Guided Dist.\\DINOv2-B} &
\makecell{Guided Dist.\\DINOv2-L} \\
\midrule
5\%  & \textbf{30.7} & 23.9 & 25.1 & 28.8 \\
10\% & \textbf{33.9} & 30.8 & 27.0 & 33.0 \\
30\% & \textbf{40.4} & 35.6 & 35.4 & 39.1 \\
\bottomrule
\end{tabular}
\end{table}

\section{Use of LLM Statement}
We leverage ChatGPT to polish the paper presentation at the sentence level.
Specifically, we provided the LLM some of the draft sentences, and asked the LLM if there is a better
version of the given sentence

%\section{More Qualitative Results}
%\label{sec:qual-ann}

\begin{figure*}
    \centering
    \includegraphics[width=0.9\linewidth]{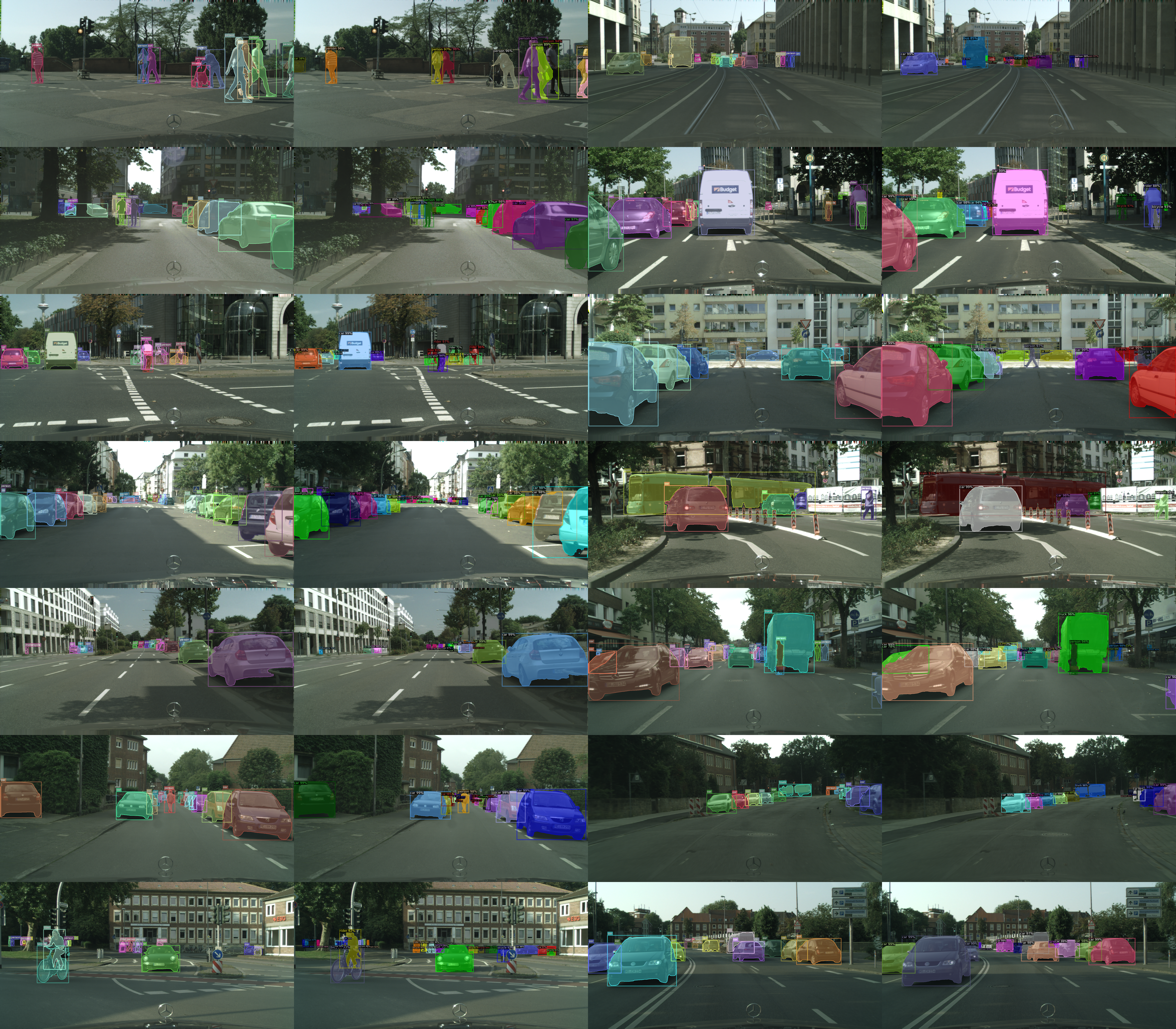}
    \caption{More Qualitative Results}
    \label{fig:more_qu}
\end{figure*}